\documentclass{article}

\usepackage{arxiv}

\usepackage[utf8]{inputenc} %
\usepackage[T1]{fontenc}    %
\usepackage{hyperref}       %
\usepackage{url}            %
\usepackage{booktabs}       %
\usepackage{amsfonts}       %
\usepackage{nicefrac}       %
\usepackage{microtype}      %
\usepackage{lipsum}         %
\usepackage{todonotes}

\newcommand{\dquote}[1]{``#1''}

\usepackage{url}            %
\usepackage{booktabs}       %
\usepackage{multirow}    
\usepackage{amsfonts}       %
\usepackage{nicefrac}       %
\usepackage{microtype}      %
\usepackage{natbib}
\usepackage{enumerate}
\usepackage{hhline}
\usepackage{makecell}
\usepackage{pifont}

\usepackage{graphicx} %
\usepackage{caption}
\usepackage{subcaption}
\usepackage{amsmath}
\usepackage{amsthm}
\usepackage{amssymb}
\usepackage{tikz}
\usepackage{xcolor}
\usetikzlibrary{arrows}

\allowdisplaybreaks

\usepackage{mathrsfs}

\usepackage{algorithm}
\usepackage{algpseudocode}
\usepackage{hyperref}
\usepackage{bm}

\allowdisplaybreaks

\newcommand{\labs}{\left\vert}
\newcommand{\rabs}{\right\vert}
\newcommand{\lnorm}{\left\Vert}
\newcommand{\rnorm}{\right\Vert}

\newcommand{\opt}{\mathrm{opt}}

\newcommand{\real}{\mathbb{R}}

\newcommand{\expect}{\mathbb{E}}

\newcommand{\indict}{\mathbb{I}}

\newtheorem{thm}{Theorem}
\newtheorem{lem}{Lemma}

\newtheorem{prop}{Proposition}

\newtheorem{oracle}{Oracle}

\newtheorem{rem}{Remark}

\usepackage[capitalize,noabbrev]{cleveref}
\crefname{thm}{Theorem}{Theorems}
\crefname{lem}{Lemma}{Lemmas}
\crefname{cor}{Corollary}{Corollaries}
\crefname{prop}{Proposition}{Propositions}
\crefname{asmp}{Assumption}{Assumptions}
\crefname{defn}{Definition}{Definitions}
\crefname{oracle}{Oracle}{Oracles}
\crefname{fact}{Fact}{Facts}
\crefname{conj}{Conjecture}{Conjectures}
\crefname{rem}{Remark}{Remarks}
\crefname{example}{Example}{Examples}
\crefname{condition}{Condition}{Conditions}
\crefname{exercise}{Exercise}{Exercises}
\crefname{algorithm}{Algorithm}{Algorithms}
\crefname{table}{Table}{Tables}
\crefname{figure}{Figure}{Figures}
\crefname{section}{Section}{Sections}
\crefname{subsection}{Section}{Sections}
\crefname{appendix}{Appendix}{Appendices}
\crefname{message}{Message}{Messages}

\renewcommand{\algorithmicrequire}{\textbf{Input:}}
\renewcommand{\algorithmicensure}{\textbf{Output:}}

\definecolor{red}{rgb}{1, 0, 0}
\newcommand{\RED}[1]{{\color{red} #1}}

\definecolor{green}{rgb}{0, 1, 0}

\definecolor{blue}{rgb}{0, 0, 1}
\newcommand{\BLUE}[1]{{\color{blue} #1}}

\definecolor{orange}{rgb}{1, 0.4, 0.0}

\input{packages/math_commands}

\title{A Note on Target Q-learning For Solving Finite MDPs with A  Generative Oracle}

\newcommand{\spa}{\text{span}}

\hypersetup{
    colorlinks=true,
    citecolor=blue,
    linkcolor=blue,
}

\author{Ziniu Li\thanks{Equal contribution.} \\
	Shenzhen Research Institute of Big Data \\ The Chinese University of Hong Kong, Shenzhen \\
	\texttt{ziniuli@link.cuhk.edu.cn} \\
	\And
	Tian Xu$^{*}$ \\ 
    National Key Laboratory for Novel Software Technology \\
    Nanjing University \\
	\texttt{xut@lamda.nju.edu.cn} \\
	\And 
	Yang Yu\thanks{Corresponding author.} \\ 
    National Key Laboratory for Novel Software Technology \\
    Nanjing University \\
	\texttt{yuy@nju.edu.cn} \\
}

\renewcommand{\headeright}{}
\renewcommand{\shorttitle}{A Note on Target Q-learning For Solving Finite MDPs with A Generative Oracle}

\begin{document}
\maketitle

\begin{abstract}
Q-learning with function approximation could diverge in the off-policy setting and the target network is a powerful technique to address this issue. In this manuscript, we examine the sample complexity of the associated target Q-learning algorithm in the tabular case with a generative oracle. We point out a misleading claim in \citep{lee2020periodic} and establish a tight analysis. In particular, we demonstrate that the sample complexity of the target Q-learning algorithm in \citep{lee2020periodic} is $\widetilde{\gO}(|\gS|^2|\gA|^2 (1-\gamma)^{-5}\varepsilon^{-2})$. Furthermore, we show that this sample complexity is improved to $\widetilde{\gO}(|\gS||\gA| (1-\gamma)^{-5}\varepsilon^{-2})$ if we can sequentially update all state-action pairs and  $\widetilde{\gO}(|\gS||\gA| (1-\gamma)^{-4}\varepsilon^{-2})$ if $\gamma$ is further in $(1/2, 1)$. Compared with the vanilla Q-learning, our results conclude that the introduction of a periodically-frozen target Q-function does not sacrifice the sample complexity.
\end{abstract}

\section{Introduction}

Q-learning is one of the most simple yet popular algorithms in the reinforcement learning (RL) community \citep{sutton2018reinforcement}. However, Q-learning suffers the divergence issue when (linear) function approximation is applied \citep{baird1995residual, tsitsiklis1997analysis}. To address this instability issue, a technique called \emph{target network} is proposed in the famous DQN algorithm \citep{mnih2015human}. In particular, DQN implements a duplication of the main Q-network (i.e., the so-called target network), which is further used to generate the bootstrap signal for updates. One important feature is that the target network is fixed over intervals. Unlike Q-learning, the learning targets do not change during an interval for DQN.  In \citep[Table 3]{mnih2015human}, it is reported that the target network contributes a lot to the superior performance of DQN.

Since then, it has been an active area of research to theoretically understand the target network technique \citep{lee2019target, lee2020periodic, fan2020theoretical, zhang2021breaking, agarwal2022online, chen2022target} and design variants based on this technique \citep{lillicrap2015ddpg, fujimoto2018td3, haarnoja2018sac, carvalho2020convergent}.
In this manuscript, we take a \dquote{sanity check}: we examine the sample complexity of target Q-learning in the tabular case with a generative oracle. We want to know whether target Q-learning sacrifices the sample complexity as it periodically freezes the target Q-function, which is believed to \dquote{may ultimately slow down training} in \citep{piche2021beyond}.

First, we revisit the target Q-learning algorithm and analysis in \citep{lee2020periodic}. In particular, once the target Q-function is fixed, this algorithm randomly picks up a state-action pair to perform the stochastic gradient descent (SGD) update. To avoid the confusion with algorithms introduced later, we call this algorithm \textbf{StoTQ-learning} (stochastic target Q-learning).  In particular, \citet{lee2020periodic} showed that the sample complexity of StoTQ-learning is $\widetilde{\gO}(|\gS|^3|\gA|^3 (1-\gamma)^{-4} \varepsilon^{-2} \log^{-1}(1/\gamma))$, where $|\gS|$ is the number of states, $|\gA|$ is the number of actions, $\gamma \in (0, 1)$ is the discount factor, and $\varepsilon \in (0, 1/(1-\gamma))$ is the error between the obtained Q-function and the optimal Q-function with respect to the $\ell_\infty$-norm. We point out that \citep{lee2020periodic} made a mis-claim that the dependence on the effective horizon is $\gO(1/(1-\gamma)^{4})$ as they ignored that $1/\log(1/\gamma) = \gO(1/(1-\gamma))$. In other word, the correct dependence on the effective horizon is $\gO(1/(1-\gamma)^{5})$. As one can see, this sample complexity suffers a poor dependence on the problem size $|\gS| \times |\gA|$. To this end, we refine the analysis in \citep{lee2020periodic} and builds a tighter upper bound on the variance of the SGD update. Consequently, we show that StoTQ-learning enjoys a sample complexity $\widetilde{\gO}(|\gS|^2|\gA|^2 (1-\gamma)^{-5} \varepsilon^{-2})$, in which the dependence on $1/(1-\gamma)$ is same with phased Q-learning \citep{kearns1999finite} and Q-learning  \citep{wainwright2019stochastic}.

Second, we demonstrate that the dependence on the problem size can be improved to $\gO(|\gS||\gA|)$ if we sequentially update state-action pairs for target Q-learning. We call such an algorithm \textbf{SeqTQ-learning} (sequential target Q-learning). Technically, SeqTQ-learning ensures that all state-action pairs can be updated after one \dquote{epoch}, which cannot be achieved by StoTQ-learning since StoTQ-learning randomly picks up a state-action pair to update during an \dquote{epoch}. In particular, the proposed modification is similar to the \dquote{random shuffling} technique in the deep-learning community, which is shown to reduce the variance compared with the original SGD update for finite-sum optimization (see \citep{mishchenko2020random} and references therein). 

Finally, we conclude that if $\gamma \in (1/2, 1)$, the sample complexity of SeqTQ-learning is improved to $\widetilde{\gO}(|\gS||\gA| (1-\gamma)^{-4} \varepsilon^{-2} )$, which is identical with the sharp sample complexity of Q-learning in \citep{li2021q}. This good result builds on the tight analysis in \citep{li2021q,agarwal2022online}. Therefore, we conclude that compared with the vanilla Q-learning, the introduction of a periodically-frozen target Q-function does not sacrifice the statistical accuracy in the tabular case with a generative oracle.

\begin{table}[t]
\caption{Sample complexity of Q-learning, phased Q-learning, and target Q-learning for solving finite MDPs with a generative oracle.}
\label{tab:summary}
\centering
\begin{tabular}{c|c|c}
Algorithm/Lower Bound & Sample Complexity & $\gamma$ \\ 
\hline 
Phased Q-learning \citep{kearns1999finite}   & $\widetilde{\gO}\lp  \frac{|\gS||\gA|}{(1-\gamma)^{5} \varepsilon^2}\rp$  & $(0, 1)$ \\
Q-learning \citep{wainwright2019stochastic} & $\widetilde{\gO}\lp  \frac{|\gS||\gA|}{(1-\gamma)^{5} \varepsilon^2}\rp$ & $(0, 1)$ \\
Stochastic Target Q-learning \citep{lee2020periodic} & $\widetilde{\gO}\lp  \frac{|\gS|^{3}|\gA|^{3}}{(1-\gamma)^{5} \varepsilon^2}\rp$  & $(0, 1)$ \\
Stochastic Target Q-learning (\cref{theorem:stochastic_target_q}) & $\widetilde{\gO}\lp  \frac{|\gS|^2|\gA|^2}{(1-\gamma)^{5} \varepsilon^2}\rp$ & $(0, 1)$ \\
Sequential Target Q-learning  (\cref{theorem:sample_complexity_of_sequential_target_q}) & $\widetilde{\gO}\lp  \frac{|\gS||\gA|}{(1-\gamma)^{5} \varepsilon^2}\rp$  & $(0, 1)$ \\ 
\hline 
Q-learning \citep{li2021q} & $\widetilde{\gO}\lp  \frac{|\gS||\gA|}{(1-\gamma)^{4} \varepsilon^2}\rp$ & $(1/2, 1)$ \\ 
Sequential Target Q-learning  (\cref{theorem:tight_sample_complexity_of_sequential_target_q})  & $\widetilde{\gO}\lp \frac{|\gS||\gA|}{(1-\gamma)^{4} \varepsilon^2} \rp$ & $(1/2, 1)$ \\ 
\hline 
Lower Bound \citep{AzarMK13} & $\Omega \lp \frac{|\gS||\gA|}{(1-\gamma)^{3} \varepsilon^2} \rp$  & $(0, 1)$
\end{tabular}
\end{table}

\section{Preliminary}

An infinite-horizon Markov Decision Process (MDP) \citep{puterman2014markov} can be describe by a tuple $\mathcal{M} = \langle \mathcal{S}, \mathcal{A}, P, r, \gamma, d_0 \rangle$. Here $\mathcal{S}$ and $\mathcal{A}$ are the state and the action space, respectively. We assume that both $\mathcal{S}$ and $\mathcal{A}$ are finite. Here $P(s^{\prime}|s, a)$ specifies the transition probability of the next state $s^\prime$ based on current state $s$ and current action $a$. The quality of each action $a$ on state $s$ is judged by the reward function $r(s, a)$. Without loss of generality, we assume that $r(s, a) \in [0, 1]$, for all $(s, a) \in \mathcal{S} \times \mathcal{A}$ throughout. Finally, $\gamma \in [0, 1)$ is a discount factor, weighting the importance of future returns, and $d_0$ specifies the initial state distribution.

From the view of the agent, it maintains a policy $\pi$ to select actions based on $\pi(a|s)$. The quality of a policy $\pi$ is measured by the state-action value function $Q^{\pi}(s, a) := \expect[ \sum_{t = 0}^{\infty} \gamma^{t} r(s_t, a_t) \mid s_0 = s, a_0  =a ] $, i.e., the cumulative discounted rewards starting from $(s, a)$. According to the theory of MDP \citep{puterman2014markov}, there exists an optimal policy $\pi^{\star}$ such that its state-action value function is optimal, i.e., $Q^{\pi^{\star}}(s, a) = \max_{\pi} Q^{\pi}(s, a)$ for all $(s, a) \in \gS \times \gA$. For simplicity, let $Q^{\star}$ denote the optimal state-action value function, which further satisfies the Bellman equation:
\begin{align*}
    Q^{\star}(s, a) = r(s, a) + \gamma \expect_{s^\prime \sim P(\cdot|s, a)}\ls \max_{a^\prime} Q^{\star}(s^\prime, a^\prime) \rs, \quad \forall (s, a) \in \gS \times \gA.
\end{align*}
Let us define the Bellman operator $\gT: \real^{|\gS||\gA|} \rar \real^{|\gS||\gA|}$,
\begin{align*}
    \gT Q(s, a) =  r(s, a) + \gamma \expect_{s^\prime \sim P(\cdot|s, a)}\ls \max_{a^\prime} Q(s^\prime, a^\prime) \rs.
\end{align*}
It is obvious that $Q^{\star}$ is the unique fixed point of $\gT$. Furthermore, $\gT$ is $\gamma$-contractive with respect to the $\ell_{\infty}$-norm:
\begin{align*}
    \forall Q_1, Q_2, \quad \lnorm \gT Q_1 - \gT Q_2 \rnorm_{\infty} \leq \gamma \lnorm Q_1 - Q_2 \rnorm_{\infty}.
\end{align*}
As a result, we can perform the fixed point iteration to solve $Q^{\star}$, which is known as the value iteration algorithm \citep{puterman2014markov}. However,
if the transition function $P$ is unknown and we have access to the sample $(s, a, r, s^\prime)$, we can define the empirical Bellman operator $\widehat{\gT}$ for a state-action value function $Q$:
\begin{align*}
    \widehat{\gT} Q(s, a) =  r(s, a) + \gamma  \max_{a^\prime} Q(s^\prime, a^\prime), \quad s^\prime \sim P(\cdot|s, a). 
\end{align*}
With the noisy estimate $\widehat{\gT}$, we can implement the stochastic approximation and such an algorithm is called Q-learning \citep{watkins1992q}. Without loss of generality, we assume the reward function is known.

\section{Algorithms and Main Results}

In this section, we investigate the sample complexity of two target Q-learning algorithms with a generative oracle (see \cref{oracle:generative_model}). In particular, the generative oracle provides a simple way of studying the sample complexity by allowing i.i.d. samples. Nevertheless, results under the setting of i.i.d. samples can be extended to the Markovian case by the coupling arguments (see for example \citep{nagaraj2020least, agarwal2022online}).

\begin{oracle}[Generative Oracle]   \label{oracle:generative_model}
Given a state-action pair $(s, a)$, the oracle returns the next state $s^\prime$ by independently sampling from the transition function $P(\cdot|s, a)$.
\end{oracle}

\subsection{Stochastic Target Q-learning}

First, we focus on the algorithm proposed in \citep{lee2020periodic} (see \cref{algo:stochastic_target_q}), which is re-named after Stochastic Target Q-learning (StoTQ-learning) for ease of presentation. In particular, we consider the simplified version where $(s, a)$ is uniformly sampled from $\gS \times \gA$ in Line 5 of \cref{algo:stochastic_target_q}. For the update rule in Line 6,  it can be viewed as one-step stochastic gradient descent of the following optimization problem:
\begin{align} \label{eq:target_q_main_objective}
   \min_{Q \in \real^{|\gS| |\gA|}} L(Q; Q_k) := \frac{1}{2|\gS||\gA|} \sum_{(s, a)} \lp \gT Q_k(s, a) - Q(s, a)  \rp^2, 
\end{align}
Specifically, the randomness comes from the sample index $(s, a)$ and the label noise in $\gT Q_k(s, a)$ because we use the empirical Bellman update $r(s, a) + \gamma \max_{a^\prime } Q_k(s^\prime, a^\prime)$. From this viewpoint, it is reasonable that $Q_{k, T}$ is close to $\gT Q_k$ as long as the step size is properly designed and the iteration number $T$ is sufficiently large. Consequently, StoTQ-learning generates a sequence $\{Q_k\}$, which performs the approximate Bellman update as the phased Q-learning algorithm (a.k.a. sampling-based value iteration) \citep{kearns1999finite}. This connection is clear in the following error bound.

\begin{algorithm}[t]
\begin{algorithmic}[1]
\caption{Stochastic Target Q-learning (StoTQ-learning)}
\label{algo:stochastic_target_q}
\Require{outer loop iteration number $K$, inner loop iteration number $T$, initialization $Q_0$, and step-sizes $\{\eta_t\}$. }
\For{iteration $k = 0, 1, \cdots, K-1$ }
\State{set $Q_{k, 0} = Q_k$.}
\For{iteration $t = 0, 1, \cdots, T-1$}
\State{randomly pick up $(s, a)$, calculate $r(s, a)$, and obtain $s^\prime$ by the generative oracle.} 
\State{update $Q_{k, t+1}(s, a) = Q_{k, t}(s, a) + \eta_t ( r(s, a) + \gamma \max_{a^\prime} {Q}_k(s^\prime, a^\prime) - Q_{k, t}(s, a) )$.}
\EndFor
\State{set $Q_{k+1} = Q_{k, T}$.}
\EndFor
\Ensure{$Q_K$}
\end{algorithmic}
\end{algorithm}

\begin{lem}[Proposition 1 of \citep{lee2020periodic}]  \label{lemma:outer_convergence_stochastic_target_q} For each outer iteration $k$, suppose that the optimization error of the inner loop satisfies that $\expect[ \Vert Q_{k} - \gT Q_{k-1} \Vert_2^2 ] \leq \varepsilon_{\opt}$ for all $k \leq K$. Then, we have that 
\begin{align*}
    \expect\ls \lnorm Q_{K} - Q^{\star} \rnorm_{\infty} \rs \leq \frac{\sqrt{\varepsilon_{\opt}}}{1-\gamma} + \gamma^{K} \expect\ls \lnorm Q_0 - Q^{\star} \rnorm_{\infty} \rs. 
\end{align*}
\end{lem}

\cref{lemma:outer_convergence_stochastic_target_q} claims that to control the final error $\varepsilon$, it is essential to ensure the optimization error is small for each inner loop. Since the learning targets are generated by a fixed variable $Q_k$ in the inner loop, SGD is stable for the optimization problem \eqref{eq:target_q_main_objective} (see \citep{bottou2018siam} and references therein). As a consequence, we expect that $\varepsilon_{\opt}$ is well-controlled. In terms of the analysis, the key is to upper bound the variance of stochastic gradients. Let $\nabla L(Q; Q_k)$ be the true gradient and  $\widetilde{\nabla} L(Q; Q_k)$ be the stochastic gradient. 

\begin{lem}[Lemma 7 of \citep{lee2020periodic}] \label{lemma:sgd_variance_loose}
For any $Q \in \real^{|\gS||\gA|}$, we have 
\begin{align*}
    \expect\ls \lnorm \widetilde{\nabla} L(Q; Q_k) \rnorm_2^2\rs \leq 8 |\gS||\gA| \lnorm \nabla L(Q; Q_k) \rnorm_2^2  + 12 \gamma^2 |\gS||\gA| \lnorm Q_k - Q^{\star} \rnorm_{\infty}^2  + \frac{18|\gS||\gA|}{(1-\gamma)^2}. 
\end{align*}
\end{lem}

Based on \cref{lemma:sgd_variance_loose}, \citet{lee2020periodic} proved the sample complexity $\widetilde{\gO}(|\gS|^{3} |\gA|^{3} (1-\gamma)^{-5} \varepsilon^{-2})$ for \cref{algo:stochastic_target_q}, in which the dependence on the problem size $|\gS||\gA|$ is inferior to algorithms like Phased Q-learning and  Q-learning (see \cref{tab:summary}). In this manuscript, we point out that the proof of \cref{lemma:sgd_variance_loose} can be improved to obtain a tighter upper bound and a better sample complexity. 

\begin{lem}[Refined Version of \cref{lemma:sgd_variance_loose}] \label{lemma:sgd_variance_tight}
For any $Q \in \real^{|\gS||\gA|}$, we have 
\begin{align*}
    \expect \ls \lnorm \widetilde{\nabla} L (Q; Q_k)  \rnorm_{2}^2 \rs \leq  |\gS||\gA| \left\|\nabla L \left(Q ; Q_{k}\right)\right\|_{2}^{2} + 6 \gamma^2 \lnorm Q_k - Q^{\star} \rnorm_{\infty} + \frac{3\gamma^2}{(1-\gamma)^2}.
\end{align*}
\end{lem}

As one can see, the upper bound in \cref{lemma:sgd_variance_tight} is better than that in \cref{lemma:sgd_variance_loose} in terms of the dependence on $|\gS||\gA|$ on the last two terms.  With \cref{lemma:sgd_variance_tight}, we arrive at a better sample complexity. 

\begin{thm}[Sample Complexity of \cref{algo:stochastic_target_q}]  \label{theorem:stochastic_target_q}
For any tabular MDP with a generative oracle, consider \cref{algo:stochastic_target_q} with the following parameters:
\begin{align*}
  Q_0 = \boldsymbol{0}, \quad   K = \gO \lp \frac{1}{1-\gamma} \log \lp \frac{1}{(1-\gamma) \varepsilon} \rp \rp , \quad T = \gO\lp \frac{|\gS|^2|\gA|^2}{(1-\gamma)^{4} \varepsilon^2} \rp, \quad \eta_t =  \frac{\eta}{\lambda + t},
\end{align*}
where $\eta = 2|\gS||\gA|$ and $\lambda = 13/2 \cdot \gamma^2 |\gS||\gA|$. Then, we have that $\expect[\Vert Q_K - Q^{\star} \Vert_{\infty}] \leq \varepsilon$. Accordingly, the number of required samples is 
\begin{align*}
KT = \widetilde{\gO} \lp \frac{|\gS|^2|\gA|^2}{(1-\gamma)^{5} \varepsilon^{2}}  \rp.    
\end{align*}

\end{thm}

Compared with the lower bound $\Omega(|\gS||\gA|(1-\gamma)^{-3}\varepsilon^{-2})$ \citep{AzarMK13}, the sample complexity shown in \cref{theorem:stochastic_target_q} is sub-optimal in the dependence on the problem size $|\gS||\gA|$ and effective horizon $1/(1-\gamma)$. In the following parts, we discuss how to improve the orders.

\subsection{Sequential Target Q-learning}

To overcome the sample barrier of StoTQ-learning, a simple yet effective approach is to sequentially update all state-action pairs (see \cref{algo:sequential_target_q}). This ensures that the optimality gap with respect to the $\ell_{\infty}$-norm is reduced after $|\gS||\gA|$ iterations, which is consistent with $\gamma$-contraction of the Bellman operator. In contrast, the uniform sampling strategy in StoTQ-learning is designed to minimize the optimality gap with respect to the $\ell_{2}$-norm. In fact, the translation between $\ell_2$-norm and $\ell_{\infty}$-norm results in the poor dependence on the problem size $|\gS||\gA|$ for StoTQ-learning.

\begin{algorithm}[t]
\begin{algorithmic}[1]
\caption{Sequential Target Q-learning (SeqTQ-learning)}
\label{algo:sequential_target_q}
\Require{outer loop iteration number $K$, inner loop iteration number $T$, initialization $Q_0$, and step-sizes $\{\eta_t\}$. }
\For{iteration $k = 0, 1, \cdots, K - 1$ }
\State{set $Q_{k, 0} = Q_{k}$.}
\For{iteration $t = 0, 1, \cdots, T-1$}
\For{each state-action pair $(s, a) \in \gS \times \gA$}
\State{calculate $r(s, a)$ and obtain $s^\prime$ by the generative oracle.}
\State{update $Q_{k, t+1}(s, a) = Q_{k, t}(s, a) + \eta_t ( r(s, a) + \gamma \max_{a^\prime} {Q}_k(s^\prime, a^\prime) - Q_{k, t}(s, a) )$.}
\EndFor
\State{Set $Q_{k+1} = Q_{k, T}$.}
\EndFor
\EndFor
\Ensure{$Q_{K}$.}
\end{algorithmic}
\end{algorithm}

\begin{thm}[Sample Complexity of \cref{algo:sequential_target_q}]    \label{theorem:sample_complexity_of_sequential_target_q}
For any tabular MDP with a generative oracle, consider \cref{algo:sequential_target_q} with the following parameters:
\begin{align*}
   Q_0 = \boldsymbol{0}, \quad  K = \gO\lp \frac{1}{1-\gamma} \log\lp \frac{1}{(1-\gamma) \varepsilon} \rp \rp, \quad T = \gO \lp \frac{1}{(1-\gamma)^{4} \varepsilon^2} \log \lp |\gS||\gA| \rp \rp, \quad \eta_t = \frac{1}{t+2}.
\end{align*}
Then, we have that $\expect[\Vert Q_K - Q^{\star} \Vert_{\infty}] \leq \varepsilon$. Accordingly, the number of required samples is 
\begin{align*}
    K \cdot T \cdot |\gS||\gA| = \widetilde{\gO} \lp \frac{|\gS||\gA|}{(1-\gamma)^{5} \varepsilon^2} \rp.
\end{align*}
\end{thm}

\begin{rem}
We note that SeqTQ-learning uses more conservative step-sizes than Q-learning. Specifically, it is a common choice that Q-learning uses the step-size $\eta_t = 1/ (1 + (1-\gamma) (t+1))$ \citep{wainwright2019stochastic, li2021q}. We explain the difference here. For each inner loop, the update rule of   SeqTQ-learning is 
\begin{align*}
    \text{SeqTQ-learning}: \quad   Q_{k, t+1} = (1 - \eta_t) Q_{k, t} + \eta_t \widehat{\gT}_{t} \RED{Q_k}.
\end{align*}
Define the error term $\Delta_{k, t} := Q_{k, t+1} - \gT Q_k$. Then, we have that
\begin{align}
  \text{SeqTQ-learning}: \quad   \Delta_{k, t+1}  &=  (1- \eta_t) \Delta_{k, t}  + \eta_t  ( \widehat{\gT}_t Q_k - \gT Q_k) . \label{eq:target_q_learning_error}
\end{align}
On the other hand, the update rule of Q-learning is 
\begin{align*}
    \text{Q-learning}: \quad   Q_{t+1} = (1 - \eta_t) Q_{t} + \eta_t \widehat{\gT_t} \BLUE{Q_t}.
\end{align*}
Define the error term $\Delta_{t} = Q_{t} - \gT Q^{\star}$. Then, we have that $\Delta_{t+1} = (1 - \eta_t) \Delta_t + \eta_t (\widehat{\gT}_t Q_t - \gT Q^{\star}) = (1 - \eta_t) \Delta_t + \eta_t (\widehat{\gT}_t Q_t - \widehat{\gT}_t Q^{\star}) + \eta_t (\widehat{\gT}_t Q^{\star} - \gT Q^{\star})$. By the $\gamma$-contraction of the empirical Bellman operator $\widehat{\gT}_t$, we obtain
\begin{align}
    \text{Q-learning}: \quad  \Delta_{t+1}  &\leq (1 - \eta_t) \Delta_{t} + \gamma \eta_t  \lnorm \Delta_{t} \rnorm_{\infty} \1 + \eta_t  ( \widehat{\gT_t} Q^{\star} - \gT Q^{\star} ), \label{eq:q_learning_error}
\end{align}
where $\leq$ holds elementwise and $\1$ is the vector filled with 1. 
We note that the variances of the noise terms in \eqref{eq:target_q_learning_error} and \eqref{eq:q_learning_error} have the same order. Furthermore, we see that the contraction coefficient does not rely on $(1-\gamma)$ in SeqTQ-learning, which explains the step-size design of SeqTQ-learning.
\end{rem}

Finally, we remark that the independence on effective horizon  can be further improved to $1/(1-\gamma)^{4}$ in the regime of $\gamma \in (1/2, 1)$. This improvement is based on the sharp analysis in \citep{li2021q, agarwal2022online}. 

\begin{thm}[Tight Sample Complexity of \cref{algo:sequential_target_q} when $\gamma > 1/2$]  \label{theorem:tight_sample_complexity_of_sequential_target_q}
For any tabular MDP with a generative oracle and $\gamma \in (1/2, 1)$, consider  \cref{algo:sequential_target_q} with the following parameters:
\begin{align*}
  Q_0 = \boldsymbol{0}, \quad    K = \gO \lp \frac{1}{1-\gamma} \log^2 \lp \frac{1}{1-\gamma} \rp \rp, \quad T = \gO \lp \frac{1}{(1-\gamma)^{3} \varepsilon^2} \log\lp \frac{K|\gS||\gA|}{\delta} \rp \rp, \quad \eta_t = \frac{1}{t+2},
\end{align*}
where $\delta \in (0, 1)$ is the failure probability. Then, with probability at least $1-\delta$, we have that $\Vert Q_K - Q^{\star} \Vert_{\infty} \leq \varepsilon$. Accordingly, the number of required samples is
\begin{align*}
    K \cdot T \cdot |\gS||\gA| = \widetilde{\gO} \lp \frac{|\gS||\gA|}{(1-\gamma)^{4} \varepsilon^2} \rp.
\end{align*}
\end{thm}

\begin{rem}
We note that the sample complexity of SeqTQ-learning in \cref{theorem:tight_sample_complexity_of_sequential_target_q} has the same order with the vanilla Q-learning \citep{li2021q} under the same setting. Compared with the lower bound in \citep{AzarMK13}, the sample complexity in \cref{theorem:tight_sample_complexity_of_sequential_target_q} is still sub-optimal in the dependence on $1/(1-\gamma)$. To further overcome the hurdle, the variance reduction scheme for sampling-based value iteration in \citep{sidford2018near, sidford2018variance} should be considered. Since the inner loop of the target Q-learning is an online version of the sampling-based value iteration\footnote{Given a target Q-function $Q_k$ to evaluate, sampling-based value iteration performs the batched update $1/T \cdot \sum_{i=1}^{T} \widehat{\gT_i} Q_k$ with $T$ i.i.d. samples in each iteration, while target Q-learning performs the online update by taking a small gradient step in each iteration. }, it is likely that the sample complexity of  target Q-learning with variance reduction can match the lower bound.
\end{rem}

\section{Conclusion}

In this manuscript, we establish the tight sample complexity for target Q-learning in the tabular setting with a generative oracle, which provides a sanity check. In particular, we conclude that compared with the vanilla Q-learning, the introduction of a periodically-frozen target Q-function does not sacrifice the sample complexity. We hope our results could provide insights for future research. 

\section*{Acknowledgements}

Ziniu Li would like to thank the helpful discussion from group members at CUHKSZ.

\bibliographystyle{abbrvnat}
\bibliography{reference}  

\newpage
\appendix
\onecolumn

\section*{\Large Appendix: A Note On Target Q-learning for Solving Finite MDPs With A Genrative Oracle}

{\footnotesize 
\tableofcontents
}

\newpage 

\section{Proofs of Main Results}

In the following proofs, we often use $c$ to denote an absolute constant, which may change in different lines. 

\subsection{Proof of Theorem \ref{theorem:stochastic_target_q}}

We prove \cref{theorem:stochastic_target_q} by following the analysis in \citep{lee2020periodic}. In particular, we obtain a stronger convergence result by \cref{lemma:sgd_variance_tight}. To make the notations consistent with \citep{lee2020periodic}, we consider the population loss is defined by 
\begin{align}   \label{eq:population_objective}
    L(Q; Q_k) = \frac{1}{2} \sum_{(s, a)} d(s, a) \lp \gT Q_k(s, a) - Q (s, a) \rp^2,
\end{align}
where $d(s, a) > 0$ assigns sampling probability for each state-action pair in Line 4 of \cref{algo:stochastic_target_q}. Specifically, we consider $d(s, a) = 1/(|\gS||\gA|)$ in \cref{theorem:stochastic_target_q}, which yields the tightest sample complexity among all sampling distributions. To facilitate later analysis, let $D \in \real^{|\gS| |\gA| \times |\gS| |\gA|}$ be the diagonal matrix of $d$. In addition, define the weighted norm by $\lnorm x \rnorm_{2, D} = \sqrt{x^{\top} D x}$. When the context is clear, we simply write $\lnorm x \rnorm_{2, D}$ by $\lnorm x \rnorm_{D}$.

\begin{lem}[Gradient Lipschitz Continuity and Strong Convexity; Lemma 6 of \citep{lee2020periodic}] \label{lemma:strong_convex_and_smooth}
The objective function $L(Q; Q_k)$ in \eqref{eq:population_objective} is $\mu$-strongly convex with $\mu = \min_{(s, a)} d(s, a)$ and $\beta$-gradient Lipschitz continuous with $\beta = \max_{(s, a)} d(s, a)$. 
\end{lem}

Based on \cref{lemma:strong_convex_and_smooth}, we arrive at the following convergence result.

\begin{prop}[Inner Loop Convergence of \cref{algo:stochastic_target_q}]
Considering \cref{algo:stochastic_target_q}, let us set $\eta_t = \eta / (\lambda + t)$ with $\eta = 2|\gS||\gA|$ and $\lambda = 13|\gS||\gA|\gamma^2/2$. Then, for all $t \geq 0$ and $k \geq 0$,  we have that 
\begin{align*}
   \expect\ls L(Q_{k, t}; Q_k) \rs :=   \expect  \ls \frac{1}{2} \lnorm Q_{k, t} - \gT Q_k \rnorm_{2, D}^{2} \rs \leq \frac{104 \beta \gamma^2}{\mu^2(1-\gamma)^2} \frac{1}{\lambda + t}.
\end{align*}
\end{prop}

\begin{proof}
By \cref{lemma:strong_convex_and_smooth}, we know that $L(Q_{k; t}; Q_k)$ is a $\beta$-smooth and $\mu$-strongly convex function. Following the typical analysis of SGD on a $\beta$-smooth and $\mu$-strongly convex function, we have that 
\begin{align*}
     L(Q_{k, t+1}; Q_k) \leq L(Q_{k, t}; Q_{k}) - \eta_t \langle \nabla L(Q_{k, t+1}; Q_k), \widetilde{\nabla} L (Q_{k, t}; Q_k) \rangle   + \frac{\beta \eta_t^2}{2} \lnorm \widetilde{\nabla} L(Q_{k, t}; Q_k)  \rnorm_2^{2}.
\end{align*}
By taking the expectation over the randomness in the stochastic gradient, we obtain that 
\begin{align*}
    \expect\ls L(Q_{k, t+1}; Q_k)  \rs  &\leq L(Q_{k, t}; Q_k) - \eta_t  \lnorm \nabla L(Q_{k, t}; Q_k) \rnorm_2^2  +  \frac{\beta \eta_t^2}{2} \expect \ls  \lnorm \widetilde{\nabla} L(Q_{k, t}; Q_k)  \rnorm_2^{2} \rs.
\end{align*}
As a corollary of \cref{lemma:bounded_variance} and \cref{lemma:boundedness_of_estimate}, we have that 
\begin{align*}
    \expect \ls \lnorm \widetilde{\nabla} L (Q_{k, t}; Q_k)  \rnorm_{2}^2 \rs &\leq \left\|\nabla L \left(Q_{k, t} ; Q_{k}\right)\right\|_{2, D^{-1}}^{2} + 6 \gamma^2 \lnorm Q_k - Q^{\star} \rnorm_{\infty} + \frac{3\gamma^2}{(1-\gamma)^2} \\
    &\leq \left\|\nabla L \left(Q_{k, t} ; Q_{k}\right)\right\|_{2, D^{-1}}^{2} + \frac{51\gamma^2}{(1-\gamma)^2} \\
    &\leq \frac{1}{\mu} \lnorm  \nabla L (Q_{k, t} ; Q_{k})  \rnorm_2^2 + \frac{51\gamma^2}{(1-\gamma)^2}.
\end{align*}
Thus, we know that 
\begin{align*}
    \expect\ls L(Q_{k, t+1}; Q_k)  \rs  &\leq L(Q_{k, t}; Q_k) - (\eta_t - \frac{\beta \eta_t^2}{2\mu})   \lnorm \nabla L(Q_{k, t}; Q_k) \rnorm_2^2  +  \frac{\beta \eta_t^2}{2} \frac{51\gamma^2}{(1-\gamma)^2} \\
    &\overset{(1)}{\leq} L(Q_{k, t}; Q_k) - (\eta_t - \frac{\beta \eta_t^2}{2\mu})  2\mu L(Q_{k, t}; Q_k)  +  \frac{\beta \eta_t^2}{2} \frac{51\gamma^2}{(1-\gamma)^2} \\
    &= L(Q_{k, t}; Q_k) - \eta_t( 2\mu  - {\beta \eta_t})  L(Q_{k, t}; Q_k)  +  \frac{\beta \eta_t^2}{2} \frac{51\gamma^2}{(1-\gamma)^2} \\
    &\leq (1 - \mu \eta_t) L(Q_{k, t}; Q_k) + \frac{\beta \eta_t^2}{2} \frac{51\gamma^2}{(1-\gamma)^2},
\end{align*}
where $(1)$ is because the strong convexity implies that $ \lnorm \nabla L(Q_{k, t}; Q_k) \rnorm_2^2 \geq 2 \mu L(Q_{k, t}; Q_k)$, and the last inequality holds when $ 0 < \eta_t \leq \mu / \beta$. Taking the expectation over the randomness before iteration $(k, t)$, we have that 
\begin{align*}
  \expect\ls L(Q_{k, t+1}, Q_k)  \rs \leq  (1 - \mu \eta_t) \expect \ls L(Q_{k, t}; Q_k) \rs  + \frac{26 \beta \gamma^2 \eta_t^2}{(1-\gamma)^2}.
\end{align*}
By choosing the diminishing step size $\eta_t = \eta / (\lambda + t)$ satisfying $\mu \eta > 1$, we obtain that 
\begin{align*}
    \expect\ls L (Q_{k, t}; Q_k) \rs \leq \frac{\nu}{\lambda + t}, \quad \forall t \geq 0, k \geq 1, 
\end{align*}
where $\nu = \max\{ \lambda \expect\ls L(Q_{k, 0}; Q_k)\rs, \eta^2  C\}$ and $C = (26 \beta \gamma^2)/ ((\mu \eta - 1 ) (1-\gamma)^2)$. Compared with the result in \citep{lee2020periodic}, $C$ is improved by a factor of $|\gS||\gA|$. 

For the initial distance, we have that 
\begin{align*}
    \expect\ls L(Q_{k, 0}; Q_k) \rs &= \expect\ls L(Q_k; Q_k) \rs = \frac{1}{2} \expect\ls \lnorm Q_k - \gT Q_k \rnorm_{2, D}^2 \rs \\
    &\leq  2 \expect\ls \lnorm Q_k - Q^{\star} \rnorm_{\infty}^2 \rs && (\text{\cref{lemma:initial_distance}}) \\
    &\leq  \frac{16}{(1-\gamma)^2}. && (\text{\cref{lemma:boundedness_of_estimate}})
\end{align*}
Thus, by choosing $\eta = 2/\mu$ and $\lambda = (13 \beta \gamma^2) / (2 \mu^2)$, we know that $\nu = (104 \beta \gamma^2) / (\mu^2 (1-\gamma)^2)$.
\end{proof}

\begin{proof}[Proof of \cref{theorem:stochastic_target_q}]
According to \cref{lemma:outer_convergence_stochastic_target_q}, if we have 
\begin{align*}
    \frac{\sqrt{\varepsilon_{\opt}}}{1-\gamma} &\leq \frac{\varepsilon}{2}, \\
    \gamma^{K} \expect\ls \lnorm Q_0 - Q^{\star} \rnorm_{\infty} \rs &\leq \frac{\varepsilon}{2},
\end{align*}
then we can sure that $\expect[ \Vert Q_K - Q^{\star} \Vert_{\infty}] \leq \varepsilon$. If we initialize $Q_0 = \boldsymbol{0}$, the second condition is satisfied when 
\begin{align} \label{eq:proof_k}
    K = \frac{1}{1-\gamma } \log\lp \frac{4}{(1-\gamma) \varepsilon} \rp .
\end{align}
For the first condition, we can ensure that $\varepsilon_{\opt}  = (1-\gamma)^2 \varepsilon^2 / 4$. Notice that for all $k \geq 1$, we have 
\begin{align*}
    \expect\ls \lnorm Q_k - \gT Q_{k-1} \rnorm_2^2  \rs \leq \frac{2}{\mu} \expect\ls L(Q_k; Q_{k-1}) \rs. 
\end{align*}
Hence, it suffices to set that 
\begin{align}  \label{eq:proof_t}
     \frac{2}{\mu} \cdot \frac{104 \beta \gamma^2}{\mu^2(1-\gamma)^2} \cdot \frac{1}{\lambda + T} &\leq \frac{(1-\gamma)^2 \varepsilon^2}{4} \quad \Longrightarrow \quad  T \geq \frac{832 \beta \gamma^2}{\mu^3 (1-\gamma)^4 \varepsilon^2}.
\end{align}
The conditions in \eqref{eq:proof_k} and \eqref{eq:proof_t} give the desired sample complexity. 
\end{proof}

\subsection{Proof of Theorem \ref{theorem:sample_complexity_of_sequential_target_q}}

\begin{lem} \label{lemma:bounded_iterate_of_seq_target_q}
For \cref{algo:sequential_target_q}, assume $\eta_t \in (0, 1)$ for all $t \geq 0$. In addition, suppose that $\lnorm Q_0 - Q^{\star} \rnorm_{\infty} \leq 1/(1-\gamma)$. Then, we have that 
\begin{align*}
    \lnorm Q_k \rnorm_{\infty} \leq \frac{1}{1-\gamma}, \quad \lnorm Q_{k, t} \rnorm_{\infty} \leq \frac{1}{1-\gamma} \quad \forall k \geq 0, t \geq 0.
\end{align*}
\end{lem}
\begin{proof}
The proof is done by a simple induction and details are therefore omitted. 
\end{proof}

\begin{lem} \label{lemma:outer_loop_sequential_target_q}
Assume that we have that $\expect[\lnorm Q_k - \gT Q_{k-1} \rnorm_{\infty}] \leq \varepsilon_{\opt}$ for all $k \leq K$. Then, we have that 
\begin{align*}
     \expect\ls \lnorm Q_K - Q^{\star} \rnorm_{\infty} \rs &\leq \frac{\varepsilon_{\opt}}{1-\gamma} + \gamma^{K} \expect\ls \lnorm Q_0 - Q^{\star} \rnorm_{\infty} \rs .
\end{align*}
\end{lem}

\begin{proof}
\begin{align*}
    \expect\ls \lnorm Q_K - Q^{\star} \rnorm_{\infty} \rs &\leq \expect\ls \lnorm Q_K - \gT Q_{K-1} + \gT Q_{K-1} - Q^{\star} \rnorm_{\infty} \rs  \\
    &\leq \expect\ls \lnorm Q_K - \gT Q_{K-1} \rnorm_{\infty} \rs + \expect\ls \lnorm \gT Q_{K-1} - Q^{\star} \rnorm_{\infty} \rs \\
    &\leq \varepsilon_{\opt} + \gamma \expect\ls \lnorm Q_{K-1} - Q^{\star} \rnorm_{\infty}  \rs  \\
    &\leq \frac{\varepsilon_{\opt}}{1-\gamma} + \gamma^{K} \expect\ls \lnorm Q_0 - Q^{\star} \rnorm_{\infty} \rs.
\end{align*}
\end{proof}

\begin{proof}[Proof of \cref{theorem:sample_complexity_of_sequential_target_q}]
Let us write down the update rule 
\begin{align*}
    Q_{k, t+1} = (1 - \eta_t) Q_{k, t} + \eta_t \widehat{\gT}_{t} Q_k,
\end{align*}
where $\widehat{\gT_t}$ is the empirical Bellman operator associated with iteration $t$. 
Define the error term $\Delta_{k, t} := Q_{k, t} - {\gT Q_k}$. Then, we have that 
\begin{align*}
    \Delta_{k, t+1} &= (1 - \eta_t) \Delta_{k, t} + \eta_t (\widehat{\gT}_t Q_k - \gT Q_k) \\
    &= \prod_{i=0}^{t} (1 - \eta_i) \Delta_{k, 0} + \sum_{i=0}^{t} \eta_i \lb \prod_{j=i+1}^{t} (1 - \eta_j) \rb ( \widehat{\gT}_i Q_k - \gT Q_k ).
\end{align*}
Let us consider the step-size $\eta_t = 1/(2 + t)$, which satisfies the condition that $(1 - \eta_t ) \leq \eta_t / \eta_{t-1}$. Accordingly, 
\begin{align*}
    \lnorm  \Delta_{k, t+1} \rnorm_{\infty} \leq \eta_t  \lnorm \Delta_{k, 0} \rnorm_{\infty} + \bigg\Vert \underbrace{\sum_{i=0}^{t} \eta_i \lb \prod_{j=i+1}^{t} (1 - \eta_j) \rb ( \widehat{\gT}_i Q_k - \gT Q_k )}_{P_{t+1}} \bigg\Vert_{\infty}.
\end{align*}
We see that the noise term $\{E_t: E_t = \widehat{\gT}_t Q_k - \gT Q_k \}$ are i.i.d. random variables with zero-mean. Furthermore, each element of $E_t$ is upper bounded by $\lnorm Q_k \rnorm_{\spa}$ and and its variance is upper bounded by $\lnorm \sigma^2(Q_k) \rnorm_{\infty}$:
\begin{align*}
    \lnorm Q_k \rnorm_{\spa} &= \max_{(s, a)} Q_k(s, a) - \min_{(s, a)} Q_k(s, a), \\
    \sigma^2(Q_k) (s, a) &= \gamma^2 \expect_{s^\prime} \ls \lp  \max_{s^\prime} Q_k(s^\prime, a^\prime) - \expect_{s^\prime}\ls \max_{a^\prime \in \gA} Q_k(s^\prime, a^\prime)  \rs  \rp^2 \rs , \\
    \lnorm \sigma(Q_k) \rnorm_{\infty} &= \sqrt{\max_{(s, a)} \sigma^2(Q_k) (s, a) }. 
\end{align*}
Define $P_{t}$ by the following recursion:
\begin{align*}
    P_{t+1} = (1 - \eta_t) P_{t} + \eta_t \lp \widehat{\gT}_t Q_k - \gT Q_k \rp  \quad \text{with} \quad P_0 = \boldsymbol{0}.
\end{align*}
This is a stationary auto-regressive process. 
By \citep[Lemma 3]{wainwright2019stochastic}, we should have that
\begin{align*}
    \expect\ls \lnorm  P_{t+1} \rnorm_{\infty} \rs \leq c\lb  \sqrt{\eta_{t}} \lnorm \sigma(Q_k) \rnorm_{\infty} \sqrt{\log \lp 2 |\gS| |\gA| \rp} +  \eta_t \lnorm Q_k \rnorm_{\spa} \log(2 |\gS||\gA|)   \rb,
\end{align*}
where $c > 0$ is an absolute constant. 
As a result, we have that 
\begin{align*}
\expect\ls \lnorm \Delta_{k, t+1} \rnorm_{\infty} \rs \leq c \sqrt{\eta_t}  \lp \lnorm \Delta_{k, 0} \rnorm_{\infty} + \lnorm \sigma(Q_k) \rnorm_{\infty} \sqrt{\log 2 |\gS| |\gA|}  + \lnorm Q_k \rnorm_{\spa} \log(2 |\gS||\gA|)  \rp. 
\end{align*}
By \cref{lemma:bounded_iterate_of_seq_target_q}, we have that 
\begin{align*}
    \lnorm \Delta_{k, 0} \rnorm_{\infty} &\leq 2 \lnorm Q_k - Q^{\star} \rnorm_{\infty} \leq \frac{2}{1-\gamma}, \\
    \lnorm \sigma(Q_k) \rnorm_{\infty} &\leq  \frac{1}{1-\gamma}, \\
    \lnorm Q_k \rnorm_{\spa} &\leq 2 \lnorm Q_k \rnorm_{\infty} \leq  \frac{2}{1-\gamma}.
\end{align*}
Consequently, we obtain that  
\begin{align*}
     \expect\ls \lnorm \Delta_{k, T} \rnorm_{\infty} \rs \leq c  \frac{1}{1-\gamma}  \sqrt{\frac{\log \lp 4 |\gS| |\gA| \rp}{T}}.
\end{align*}
According to \cref{lemma:outer_loop_sequential_target_q}, it suffices to consider that 
\begin{align*}
    \varepsilon_{\opt} = \frac{(1-\gamma)\varepsilon}{2}, \quad K = \frac{1}{1-\gamma} \log \lp \frac{4}{(1-\gamma) \varepsilon} \rp.
\end{align*}
This further implies that 
\begin{align*}
    \frac{c}{1-\gamma} \sqrt{\frac{\log \lp 4 |\gS| |\gA| \rp}{T}} \leq \frac{(1-\gamma) \varepsilon}{2} \quad \Longrightarrow \quad T \geq \frac{c \log \lp 4 |\gS| |\gA| \rp}{(1-\gamma)^4 \varepsilon^2}.
\end{align*}
Hence, the total sample complexity is 
\begin{align*}
    T \cdot K \cdot |\gS| |\gA| = \widetilde{\gO}\lp \frac{|\gS| |\gA|}{(1-\gamma)^{5}\varepsilon^2} \rp. 
\end{align*}
\end{proof}

\subsection{Proof of Theorem \ref{theorem:tight_sample_complexity_of_sequential_target_q}}

\begin{proof}[Proof of \cref{theorem:tight_sample_complexity_of_sequential_target_q}]

Following the same steps in the proof of \cref{theorem:sample_complexity_of_sequential_target_q},  we have that 
\begin{align*}
    \Delta_{k, t+1} = \eta_t   \Delta_{k, 0} + \eta_t \sum_{i=0}^{t} \lp \widehat{\gT}_i Q_k - \gT Q_k \rp.
\end{align*}
For our purpose, let us define $\gF_t$ be the sigma-algebra of all state-action-reward pairs generated before iteration $t$. Then, for all $(s, a) \in \gS \times \gA$, we have that 
\begin{align*}
    \expect\ls \widehat{\gT}_t Q_k(s, a) - \gT Q_k(s, a) \mid \gF_t \rs = 0, \quad \forall t \geq 0, 
\end{align*}
Furthermore, for all $(s, a) \in \gS \times \gA$ and $t \geq 0$, we have that 
\begin{align*}
    \expect\ls \lp \widehat{\gT}_t Q_k(s, a) - \gT Q_k(s, a) \rp^2 | \gF_t \rs = \sigma^2(Q_k)(s, a) = \gamma^2 \Var_P (Q_k) (s, a), 
\end{align*}
where $\Var_P (Q_k) (s, a) = \expect_{s^\prime} [ \lp  \max_{a^\prime} Q_k(s^\prime, a^\prime) - \expect_{s^\prime}\ls \max_{a^\prime} Q_k(s^\prime, a^\prime)  \rs  \rp^2 ]$.
Consider the sum of conditional variances:
\begin{align*}
    W_{t+1}(s, a) &:= \sum_{i=0}^{t}  \expect\ls \lp \widehat{\gT}_t (Q_k)(s, a) - \gT Q_k(s, a) \rp^2 \mid \gF_{t}   \rs \\
    &= (t+1) \gamma^2 \Var_P (Q_k) (s, a).
\end{align*}
According to \cref{lemma:bounded_iterate_of_seq_target_q}, we have that $   \lnorm Q_k \rnorm_{\infty} \leq 1/(1-\gamma)$.  Now, we can apply \cref{lemma:freedman_inequality} with $R = 1/(1-\gamma)$ and $\sigma^2 = (t+1) \gamma^2 \Var_P (Q_k) (s, a)$ and $K = 1$ to obtain that for any $(s, a) \in \gS \times \gA$, with probability $1 - \delta$,  we have 
\begin{align*}
  \labs  \eta_t \sum_{i=0}^{t} \lp \widehat{\gT}_i (Q_k)(s, a) - \gT Q_k(s, a) \rp  \rabs  \leq c \eta_t \lp \sqrt{(t+1) \gamma^2 \Var_P (Q_k) (s, a) \log \lp \frac{2}{\delta} \rp } + \frac{1}{1-\gamma} \log \lp \frac{2}{\delta} \rp \rp.
\end{align*}
In summary, we have that with probability $1-\delta/K$, we have that 
\begin{align}   \label{eq:recursion}
    Q_{k+1} = \gT Q_k + E_{k},
\end{align}
where $E_k$ is an error term satisfying that 
\begin{align*}
   | E_{k}(s, a)| &\leq \eta_{T-1} \lnorm Q_k - \gT Q_k \rnorm_{\infty} + c \eta_{T-1}  \sqrt{T \gamma^2 \Var_P (Q_k) (s, a) \log \lp \frac{2K|\gS||\gA|}{\delta} \rp }  \\
    &\quad +  c \eta_{T-1}\frac{1}{1-\gamma} \log \lp \frac{2K|\gS||\gA|}{\delta} \rp  \\
    &\leq c \frac{1}{T+1} \frac{1}{1-\gamma} + c \frac{1}{T+1}  \sqrt{T \gamma^2 \Var_P (Q_k) (s, a) \log \lp \frac{2K|\gS||\gA|}{\delta} \rp } \\
    &\quad + c \frac{1}{T+1} \frac{1}{1-\gamma} \log \lp \frac{2K|\gS||\gA|}{\delta} \rp  \\
    &\leq \alpha_{T}  + c \sqrt{\frac{1}{T} \gamma^2 \Var_P (Q_k) (s, a) \log \lp \frac{2K|\gS||\gA|}{\delta} \rp  },
\end{align*}
where
\begin{align*}
    \alpha_{T} = c \frac{1}{T(1-\gamma)} + c \frac{1}{T(1-\gamma)} \log \lp \frac{2K|\gS||\gA|}{\delta} \rp.
\end{align*}
Now, we see that the recursion in \eqref{eq:recursion} has the same form with that in \citep{agarwal2022online}.  Following the same steps in \citep{agarwal2022online}, when $T \geq 1/\log(K|\gS||\gA|/\delta)$, with probability at least $1-\delta$, we have 
\begin{align*}
\lnorm \Delta_{K} \rnorm_{\infty} \leq c \ls \frac{1}{T(1-\gamma)^{3}} \log \lp \frac{K|\gS||\gA|}{\delta} \rp + \frac{\alpha_T + \gamma^{L}}{1-\gamma} +  \sqrt{ \frac{1}{T(1-\gamma)^3} \log\lp \frac{K|\gS||\gA|}{\delta} \rp }  \rs,
\end{align*}
where $L = cK/\log(1/(1-\gamma))$. Thus, by  
\begin{align*}
    K = \widetilde{\gO}\lp \frac{1}{1-\gamma} \rp, \quad T = \widetilde{\gO} \lp \frac{1}{(1-\gamma)^{3}} \rp,
\end{align*}
with probability $1-\delta$, we obtain that $\Vert Q_K - Q^{\star} \Vert_{\infty} \leq \varepsilon$.

\end{proof}
\section{Technical Lemmas}

\begin{lem}[Upper Bound of Stochastic Gradient Variance]
\label{lemma:bounded_variance}
In iteration $k$ and timestep $t$,
\begin{align*}
    \expect \ls \lnorm \widetilde{\nabla} L (Q_{k, t}; Q_k)  \rnorm_{2}^2 \rs \leq \left\|\nabla L \left(Q_{k, t} ; Q_{k}\right)\right\|_{2, D^{-1}}^{2} + 6 \gamma^2 \lnorm Q_k - Q^{\star} \rnorm_{\infty} + \frac{3\gamma^2}{(1-\gamma)^2}.
\end{align*}
\end{lem}

\begin{proof}
Conditioned on $ Q_k, Q_{k,t}$, we have 
\begin{align*}
    &\quad \expect \ls \lnorm \widetilde{\nabla} L (Q_{k, t}; Q_k)  \rnorm_{2}^2  \rs \\
    &= \expect \ls \sum_{(s, a) \in \gS \times \gA} \indict\{ (S, A) = (s, a) \}^2 \lp Q_{k, t}(s, a) - r(s, a) - \gamma  \max_{a^\prime} Q_k (S^\prime, a^\prime) \rp^2  \rs
    \\
    &= \sum_{(s, a) \in \gS \times \gA} \expect \ls \indict\{ (S, A) = (s, a) \}^2 \lp Q_{k, t}(s, a) - r(s, a) - \gamma  \max_{a^\prime} Q_k (S^\prime, a^\prime) \rp^2  \rs
    \\
    &\overset{(1)}{=} \sum_{(s, a) \in \gS \times \gA} \expect \ls \expect \ls \indict\{ (S, A) = (s, a) \}^2 \lp Q_{k, t}(s, a) - r(s, a) - \gamma  \max_{a^\prime} Q_k (S^\prime, a^\prime) \rp^2 \mid S, A \rs  \rs
    \\
    &\overset{(2)}{=} \sum_{(s, a) \in \gS \times \gA} \expect \ls \indict\{ (S, A) = (s, a) \}^2 \expect \ls  \lp Q_{k, t}(s, a) - r(s, a) - \gamma  \max_{a^\prime} Q_k (S^\prime, a^\prime) \rp^2 \mid S, A \rs  \rs
    \\
    &\overset{(3)}{=} \sum_{(s, a) \in \gS \times \gA} \expect \ls \indict\{ (S, A) = (s, a) \}^2 \expect \ls  \lp Q_{k, t}(s, a) - r(s, a) - \gamma  \max_{a^\prime} Q_k (S^\prime, a^\prime) \rp^2 \mid s, a \rs  \rs. 
\end{align*}
Equality $(1)$ follows the Tower property, equality $(2)$ follows that $\indict\{ (S, A) = (s, a) \}$ is determined by $S, A$ and equality $(3)$ holds because of the indicator function. We first consider the term $\expect_{S^\prime \sim P (\cdot|s, a)} [ ( Q_{k, t}(s, a) - r(s, a) - \gamma  \max_{a^\prime} Q_k (S^\prime, a^\prime) )^2 ]$. With $\expect [X^2] = (\expect [X])^2 + \Var [X]$, we obtain
\begin{align*}
    &\quad \expect_{S^\prime \sim P (\cdot|s, a)} \ls  \lp Q_{k, t}(s, a) - r(s, a) - \gamma  \max_{a^\prime} Q_k (S^\prime, a^\prime) \rp^2 \rs
    \\
    &= \lp Q_{k, t}(s, a) - \gT Q_k (s, a) \rp^2 + \Var \ls Q_{k, t}(s, a) - r(s, a) - \gamma  \max_{a^\prime} Q_k (S^\prime, a^\prime) \rs
    \\
    &= \lp Q_{k, t}(s, a) - \gT Q_k (s, a) \rp^2 + \Var \ls   \widehat{\gT} Q_k (s, a) \rs
    \\
    &\leq \lp Q_{k, t}(s, a) - \gT Q_k (s, a) \rp^2 + 6 \gamma^2 \lnorm Q_k - Q^{\star} \rnorm_{\infty}^2 + \frac{3\gamma^2}{(1-\gamma)^2} &&{\text{(\cref{lemma:variance_of_empirical_bellman_operator}})}. 
\end{align*}
Then we have that
\begin{align*}
    &\quad \expect \ls \lnorm \widetilde{\nabla} L (Q_{k, t}; Q_k)  \rnorm_{2}^2 \rs
    \\
    &\leq \sum_{(s, a) \in \gS \times \gA} \expect \ls \indict\{ (S, A) = (s, a) \}^2 \lp \lp Q_{k, t}(s, a) - \gT Q_k (s, a) \rp^2 + 6 \gamma^2 \lnorm Q_k - Q^{\star} \rnorm_{\infty}^2 + \frac{3\gamma^2}{(1-\gamma)^2} \rp \rs
    \\
    &= \sum_{(s, a) \in \gS \times \gA} \lp \lp Q_{k, t}(s, a) - \gT Q_k (s, a) \rp^2 + 6 \gamma^2 \lnorm Q_k - Q^{\star} \rnorm_{\infty}^2 + \frac{3\gamma^2}{(1-\gamma)^2} \rp \expect \ls \indict\{ (S, A) = (s, a) \}^2  \rs. 
\end{align*}
Notice that $\indict\{ (S, A) = (s, a) \} \sim \text{Ber} (d(s, a))$ and $\expect \ls \indict\{ (S, A) = (s, a) \}^2  \rs = d(s, a)$. Then we have that
\begin{align*}
    &\quad \expect \ls \lnorm \widetilde{\nabla} L (Q_{k, t}; Q_k)  \rnorm_{2}^2 \rs
    \\
    &\leq \sum_{(s, a) \in \gS \times \gA} \lp \lp Q_{k, t}(s, a) - \gT Q_k (s, a) \rp^2 + 6 \gamma^2 \lnorm Q_k - Q^{\star} \rnorm_{\infty}^2 + \frac{3\gamma^2}{(1-\gamma)^2} \rp d(s, a)
    \\
    &= \sum_{(s, a) \in \gS \times \gA} \frac{1}{d(s, a)} d(s, a)^2 \lp Q_{k, t}(s, a) - \gT Q_k (s, a) \rp^2 + 6 \gamma^2 \lnorm Q_k - Q^{\star} \rnorm_{\infty}^2 + \frac{3\gamma^2}{(1-\gamma)^2}
    \\
    &= \left\|\nabla L \left(Q_{k, t} ; Q_{k}\right)\right\|_{2, D^{-1}}^{2} + 6 \gamma^2 \lnorm Q_k - Q^{\star} \rnorm_{\infty}^2 + \frac{3\gamma^2}{(1-\gamma)^2}.
\end{align*}
\end{proof}

\begin{lem}[Initial Distance]
\label{lemma:initial_distance}
Before the inner loop starts, we have that 
\begin{align*}
      \lnorm \gT Q_k - Q_{k, 0} \rnorm_{D}^2  \leq 4 \lnorm Q_k - Q^{\star} \rnorm_{\infty}^2.
\end{align*}
\end{lem}

\begin{proof} 
We have that 
\begin{align*}
   \lnorm \gT Q_k - Q_{k, 0} \rnorm_{D}^2  &=  \lnorm \gT Q_k - Q_k \rnorm_{D}^2  \\
    &=  \lnorm \gT Q_k - \gT Q^{\star} + \gT Q^{\star} - Q_k \rnorm_{D}^2  \\
    &= \sum_{(s, a)} d(s, a) \lp \gT Q_k(s, a) - \gT Q^{\star}(s, a) + \gT Q^{\star}(s, a) - Q_k(s, a)  \rp^2 \\
    &\leq \sum_{(s, a)} d(s, a) \ls 2 \lp \gT(Q_k)(s, a) - \gT Q^{\star}(s, a) \rp^2 + 2 \lp \gT Q^{\star}(s, a) - Q_k(s, a) \rp^2 \rs \\
    &\leq \sum_{(s, a)} d(s, a) \ls 2 \gamma^2 \lnorm Q_k - Q^{\star} \rnorm_{\infty}^2 + 2 \lnorm Q_k - Q^{\star} \rnorm_{\infty}^2 \rs \\
    &\leq 4 \lnorm Q_k - Q^{\star} \rnorm_{\infty}^2.
\end{align*}
\end{proof}

\begin{lem}[Variance of $\widehat{\gT}(Q_k)$]
\label{lemma:variance_of_empirical_bellman_operator}
For each $k$, we have that 
\begin{align*}
    \Var[ \widehat{\gT} Q_k(s, a)] \leq 6 \gamma^2 \lnorm Q_k - Q^{\star} \rnorm_{\infty}^2 + \frac{3\gamma^2}{(1-\gamma)^2}.
\end{align*}
\end{lem}

\begin{proof}
\begin{align*}
    &\quad \Var[ \widehat{\gT} Q_k(s, a)] \\
    &= \expect\ls \lp \widehat{\gT}Q_k(s, a) - \gT Q_k(s, a) \rp^2 \rs \\
    &= \expect\ls \lp \widehat{\gT}Q_k(s, a) - \gT Q_k(s, a) - \widehat{\gT}Q^{\star}(s, a) + \widehat{\gT}Q^{\star}(s, a) - \gT Q^{\star}(s, a) + \gT Q^{\star} (s, a)  \rp^2 \rs \\
    &\leq 3\expect\ls \lp \widehat{\gT} Q_k(s, a) - \widehat{\gT} Q^{\star}(s, a) \rp^2 \rs + 3 \expect\ls \lp \gT Q_k (s, a) - \gT Q^{\star} (s, a) \rp^2 \rs + 3\expect\ls \lp \widehat{\gT}Q^{\star}(s, a) - \gT Q^{\star}(s, a) \rp^2 \rs \\
    &\leq 3 \gamma^2 \lnorm Q_k - Q^{\star} \rnorm_{\infty}^2 + 3 \gamma^2 \lnorm Q_k - Q^{\star} \rnorm_{\infty}^2 + 3\frac{\gamma^2}{(1-\gamma)^2} \\
    &= 6 \gamma^2 \lnorm Q_k - Q^{\star} \rnorm_{\infty}^2 + \frac{3\gamma^2}{(1-\gamma)^2}.
\end{align*}
\end{proof}

\begin{lem}[Boundedness of Estimate; Lemma 8 of \citep{lee2020periodic}]
\label{lemma:boundedness_of_estimate}
Suppose that $\mathbb{E} [\left\|Q_{i}- \gT Q_{i-1}\right\|_{2}^{2} ] \leq \varepsilon_{\opt}, \forall i \leq k \text { and } \varepsilon_{\opt} \leq(1-\gamma)^{2}$. Then, we have that 
\begin{align*}
    \mathbb{E}\left[\left\|Q_{k}-Q^{\star}\right\|_{\infty}^{2}\right] \leq \frac{8}{(1-\gamma)^{2}}.
\end{align*}
\end{lem}

\begin{lem}[Freedman's Inequality]  \label{lemma:freedman_inequality}
Suppose that $Y_n = \sum_{k=1}^{n} X_k \in \real$, where $\{X_k\}$ is a real-valued scalar sequence obeying 
\begin{align*}
    |X_k| \leq R, \quad \text{and} \quad \expect\ls X_k \mid \{X_j\}_{j: j < k} \rs = 0 \quad \forall k \geq 1. 
\end{align*}
Define 
\begin{align*}
    W_n := \sum_{k=1}^{n} \expect_{k-1}[X_k^2],
\end{align*}
where the expectation $\expect_{k-1}$ is conditional on $\{X_j\}_{j: j < k}$. Then, for any given $\sigma^2 \geq 0$, we have that 
\begin{align*}
    \sP\lp |Y_n| \geq \tau \text{ and } W_n \leq \sigma^2  \rp \leq 2 \exp\lp - \frac{\tau^2}{\sigma^2 + R \tau / 3}  \rp.
\end{align*}
In addition, if $W_n \leq \sigma^2$ almost surely, for any positive integer $K \geq 1$, we have that 
\begin{align*}
    \sP \lp |Y_n| \leq \sqrt{8 \max \lb W_n, \frac{\sigma^2}{2^K} \rb \log \frac{2K}{\delta} } + \frac{4}{3} R \log \frac{2K}{\delta} \rp \geq 1 - \delta .
\end{align*}
\end{lem}

\end{document}